\documentclass[twoside]{article}

\usepackage[accepted]{aistats2022}
\usepackage[square]{natbib}
\usepackage{blindtext}
\usepackage{mathtools}
\usepackage{amssymb,xcolor}
\usepackage{bm}
\usepackage{url}
\usepackage{graphicx}
\graphicspath{ {./images/} }

\usepackage{booktabs}
\usepackage{listings}
\usepackage{amsthm}
\usepackage{caption}
\usepackage{subfig}

\usepackage{siunitx}
\newtheorem{defi}{Definition}

\newtheorem{prop}{Proposition}

\DeclareMathOperator*{\argmax}{argmax}

\begin{document}

%

%

\twocolumn[

\aistatstitle{The Causal Loss: Driving Correlation to Imply Causation}

\aistatsauthor{ Moritz Willig \And Matej Zečević \And Devendra Singh Dhami \And Kristian Kersting }

\aistatsaddress{ TU Darmstadt \And TU Darmstadt \And TU Darmstadt \And TU Darmstadt, hessian.AI} ]

\begin{abstract}
Most algorithms in classical and contemporary machine learning focus on correlation-based dependence between features to drive performance. Although success has been observed in many relevant problems, these algorithms fail when the underlying causality is inconsistent with the assumed relations. We propose a novel model-agnostic loss function called \emph{Causal Loss} that improves the interventional quality of the prediction using an intervened neural-causal regularizer. In support of our theoretical results, our experimental illustration shows how causal loss bestows a non-causal associative model (like a standard neural net or decision tree) with interventional capabilities.
\end{abstract}

\section{Introduction}
One of the earliest investigation on developing a theory of causality can be found in the work of Arsitotle \citep{evans1959causality,falcon2006aristotle} which brings out his view that providing the relevant cause(s) is a necessary and sufficient condition for atoing a scientific explanation for a given problem. Due to the difficulty in obtaining causation and the early formulation of the De Moivre–Laplace theorem \citep{de1718doctrine}, a special case of central limit theorem, causation was mostly reduced to a special case of correlation, especially in the works of \cite{galton1886regression} and \cite{pearson1892grammar}. On the other side of the spectrum, works by \cite{hume2003treatise}, \cite{wright1934method} and \cite{mackie1974cement} cite causality as the most important factor in study of human reasoning.

This debate of correlation vs causation is also central to machine learning with most algorithms focusing on correlation-based dependence between features to drive performance. Although this has resulted in impressive performance in variety of tasks such as computer vision \citep{krizhevsky2012imagenet,goodfellow2014generative} and natural language processing \citep{hochreiter1997long,vaswani2017attention}, machine learning leaves a lot to be desired when dealing with problems where underlying causality is inconsistent with the assumed relations \citep{altman2015points}. For example, deep neural networks for image classification are heavily dependent on the cross-correlation (same as convolution sans kernel flipping) between features but perform poorly when tasked with extracting causal information. We argue that deep networks are rather about computation depending on what one wants to compute. Thus there has been a lot of argument lately, to combine the power of neural models and causality \citep{vasudevan2021off,scholkopf2021toward,xia2021causal}.

We present the \emph{Causal Loss}, which allows us to explore how much correlation based models can be converted to include underlying causal relationships in the data. We show that the proposed loss function, by using an interventional density estimator, is able to elevate classical discriminative machine learning models from rung 1 to rung 2 of the Pearl causal hierarchy (PCH) \citep{pearl2009causality}. In the same vein, we propose \underline{Ca}usal \underline{S}um-\underline{P}roduct \underline{N}etworks (CaSPN) that estimate a conditional interventional distribution and then prove that CaSPNs can be used as a causal loss. We show that the combination of a classical and causal loss can be used to incorporate causal information into training machine learning models while paying tribute to possible uncertainties about the underlying causal structure. This might be of interest in scenarios where the underlying structural causal model can not be approximated to a sufficient degree or when we want to express concerns about its quality.

We make the following contributions: (1) We propose a loss function that drives correlation based models to be more causal. (2) We introduce a causal probabilistic circuit, CaSPN and prove that it can be used as a causal loss. (3) We show that both differentiable and non-differentiable models benefit from using causal loss function and show that our causal loss gives a strong feedback for training neural classifier. (4) As an important result, we show that using a causal loss for learning a decision tree can result in generation of compact trees while matching the performance of a heavily parameterized classical decision tree.




\section{Structural Causal Models (SCMs)}

The goal of this section is to develop the general notion of a causal loss, as well as to present an applicable neural causal loss. For the general definition we start with describing the required properties of the loss and follow up with the resulting implications for model learning. To present an applicable causal loss we choose a novel type of neural causal model and proof that it meets the requirements to act as a causal loss term.

To capture the causal nature of a data generating process, a model needs to be aware of the underlying causal relationships and its behaviour to changes to that process. According to \cite{peters2017elements} causal systems can be expressed by structural causal models (SCMs). SCMs $\mathfrak{C}$ describe the density $p(\mathbf{x})$ of a joint distribution $P_\mathbf{X}$ over random variables $\mathbf{X} = (X_1,...,X_d)$. The relationships between the variables is captured by a directed acyclic graph $\mathcal{G} = (\mathbf{V}, \varepsilon)$ with a set of vertices $\mathbf{V} \in {1,...,d}$ and directed edges $\varepsilon \in \mathbf{V}^2$. Each variable is described by a structural assignment: $X_i = f_i(\textbf{PA}_i, N_i)$. Where $\textbf{PA}_i$ are the parents of $X_i$ in the graph and $N_i$ are jointly independent noise terms.

We define an intervention $do(Z_i = z_i)$ on the SCM, according to \cite{pearl2009causality}, as the replacement of the structural equation of $X_i$ by a new equation $X_i = \hat{f}_i(\hat{\textbf{PA}}_i, N_i)$. An SCM not only provides an observational distribution, but also entails an interventional distribution that reflects changes in probability density with regard to applied interventions. Having the knowledge and incorporating information about the underlying causal graph and the effects of interventional changes distinguishes causal models from purely correlation based methods and lifts them onto rung two of the Pearl Causal Hierarchy (PCH) (\citep{bareinboim2020pearl}, Def. 9).

\section{The Causal Loss}

The idea of our causal loss is to asses the likelihood of some prediction $\hat{Y}$ given by a model in an interventional setting.
This approach not only allows the transfer of information captured by (possibly intractable) causal models into tractable neural approaches, but it can be shown empirically that applying a causal loss helps in general to improve model performance. Independent of the scenario, we demonstrate that our causal loss can help to bridge the gap between causal and correlation-based models such as the majority of neural networks.

We require the underlying density estimator to operate at least on rung two of Pearl's Causal Hierarchy to guarantee causal nature of our loss. Utilizing Pearls do-calculus \citep{pearl2009causality} we, therefore, require the density estimator to operate on interventional distributions $P(\bm{Y} | \bm{X}, do(Z_j = z_j))$.

\begin{defi} (Causal Loss). A causal loss $\mathfrak{L}^C(\hat{Y}, | \mathfrak{C}, X, do(Z_j = z_j))$ measures the probability $P_\mathfrak{C}(\bm{Y} | \bm{X}, do(Z_j = z_j))$ for some configuration $\hat{Y}$ permitted by a structural causal model $\mathfrak{C}$ under interventions $do(Z_j = z_j)$ given $X$.
\end{defi}

The requirement of incorporating interventional information on the graph lift the loss to rung two of the PCH and makes them causal.

\subsection{Causal Sum-Product Networks}

\begin{figure*}
\centering
\subfloat{{\includegraphics[width=0.13\linewidth]{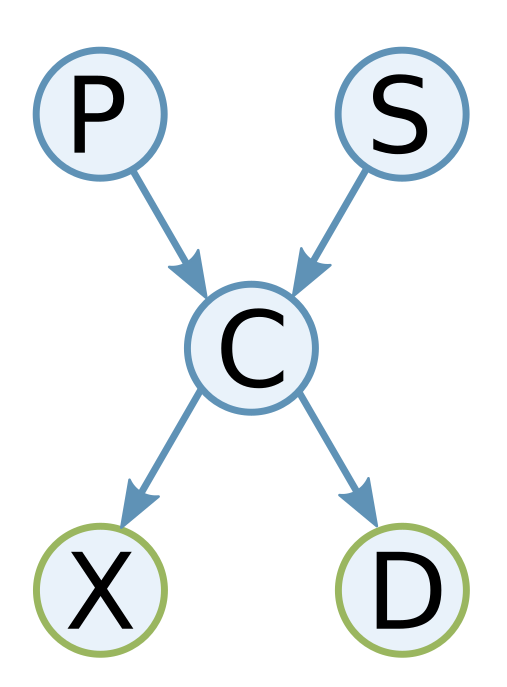}}}
\hspace{1mm}
\subfloat{{\includegraphics[width=0.17\linewidth]{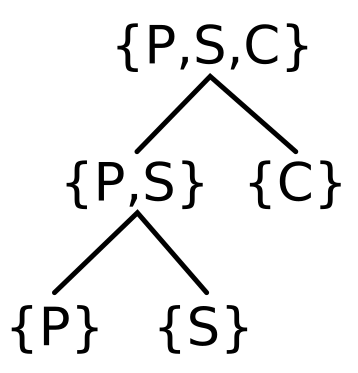}}}
\hspace{1mm}
\subfloat{{\includegraphics[width=0.43\linewidth]{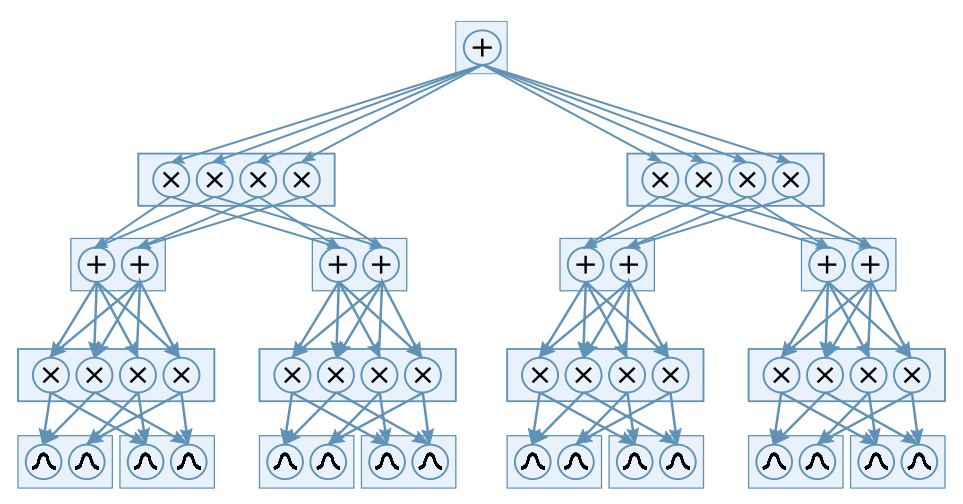}}}
\hspace{5mm}
\vrule
\hspace{3mm}
\subfloat{{\includegraphics[width=0.13\linewidth]{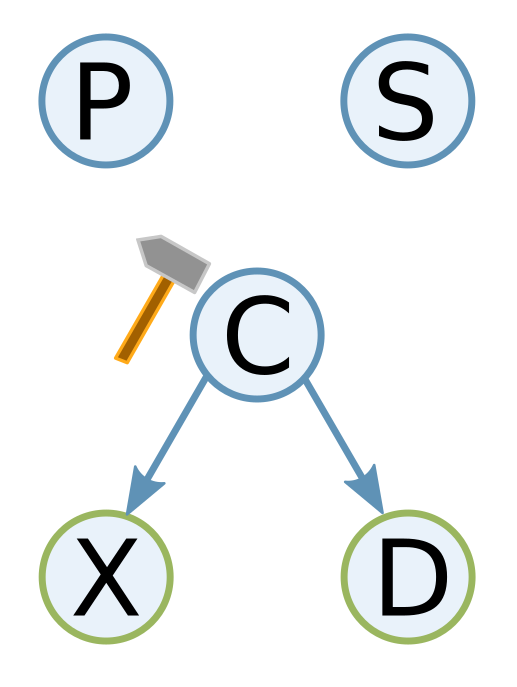}}}


\caption{Constructing  a CaSPN from a structural causal model (SCM) in four steps. From left to right, we start with a SCM (left). Target variables are marked green, and input variables to the SPN are marked blue. Now, a region graph is constructed from the SCM (middle left). The SPN (exemplary) is constructed from the region graph (middle right). Finally, training data is generated by intervening on the input variables (right). (Best viewed in color)}
\label{caSPNConstruction}
\end{figure*}

Since being proposed, Sum-Product Networks (SPNs) \citep{Poon2011} have been continually extended. An important step towards competitive SPNs models was the introduction of random sum-product networks \citep{Peharz2020a}, which allowed SPNs to compete with neural models on a range of computer vision tasks. A further addition was the introduction of conditional SPNs (cSPNs) \citep{shao2020conditional}, which allow to estimate a conditional probability $p(Y|X)$ by utilizing a function approximator $\psi := f(\cdot)$, depending on the conditional, to estimate the weights of the SPN gating nodes. In the light of causality \cite{Zecevic2021a} utilized cSPNs to condition  on interventional information. The resulting interventional sum-product networks (iSPNs) therefore estimate the probability of an intervened distribution $p(\mathbf{Y}|do(Z_i = z_i))$.

To devise the causal loss, we now extend iSPNs towards \textbf{Causal Sum-Product Networks} (CaSPNs). Specifically, CaSPNs extend iSPNs by reintroducing conditional variables $Y$ back into the SPN structure. CaSPNs therefore estimate an  interventional distribution $p_\mathcal{G}(\mathbf{Y}|\mathbf{X}, do(Z_i = z_i))$ conditioned on observational data. As a consequence the inputs of the function approximator $f$ is adapted to accept the intervened causal structure $\mathcal{G}$ in conjunction with the conditional variables $X$ to estimate $\psi := f(\mathcal{G}, \mathbf{X})$. The structure and set of parameters $\psi$ of the SPN remain unchanged.

\textbf{CaSPNs as causal losses}.
WE now argue that using a CaSPN as causal losss is a sensible idea. First, we show that CaSPNs are expressive enough. Then, we formally show that CaSPNs are causal losses. 
\begin{prop} (Expressivity)
Assuming autonomy and invariance, a CaSPN $m(\textbf{G},\textbf{D})$ is able to identify any conditional interventional distribution $p_\mathcal{G}(\mathbf{Y}_i = \mathbf{y}_i | \mathbf{X}_i = \mathbf{x}_i, do(\mathbf{Z}_j = \mathbf{z}_j))$, permitted by a SCM $\mathfrak{C}$ through interventions, with knowledge of the mutilated causal graph $\mathbf{D}$ generated from the intervened SCMs by modeling the conditional distribution $p_\mathcal{\hat{G}}(\mathbf{Y}_i = \mathbf{Y}_i|\mathbf{X}_i = \mathbf{x}_i, \mathbf{Z}_j = \mathbf{z}_j)$.
\end{prop}
\begin{proof}
 It follows from the definition of the do-calculus \citep{pearl2009causality} that
 $p_\mathcal{G}(\mathbf{Y}_i = \mathbf{y}_i | \mathbf{X}_i = \mathbf{x}_i, do(\mathbf{Z}_j = \mathbf{z}_j)) = p_\mathcal{\hat{G}}(\mathbf{Y}_i = \mathbf{y}_i | \mathbf{X}_i = \mathbf{x}_i, \mathbf{Z}_j = \mathbf{z}_j)$ where $\mathcal{\hat{G}}$ is the mutilated causal graph according to intervention $do(\mathbf{Z}_j = \mathbf{z}_j)$. The resulting joint probability, given the mutilated graph $\mathcal{\hat{G}}$, can be approximated by a SPN \citep{Poon2011}.
\end{proof}
\begin{prop} A CaSPN $m(\mathbf{G},\mathbf{X})$ is a causal loss.
\end{prop}
\begin{proof}
  Because CaSPNs are CSPNs, they can approximate any conditional probability $p_\mathcal{\hat{G}}(\mathbf{Y}|\mathbf{X}, \mathbf{Z}_j = \mathbf{z}_j)$ entailed by an SCM $\mathfrak{C}$ \citep{shao2020conditional}. From Proposition 1 it follows that CaSPNs are acting on interventional distributions according Pearls do-calculus \citep{pearl2009causality}. This places them on rung two of the Pearl Causal Hierarchy \citep{pearl2019seven}. Approximating the conditional distribution $p^\mathfrak{C}_\mathcal{\hat{G}}(\mathbf{Y}|\mathbf{X}, \mathbf{Z}_j = \mathbf{z}_j)$ of a structural causal model under intervention fulfills the definition of a causal loss.
\end{proof}
That is, we now have a candidate for a causal loss at hand. Let us now show how to use a CaSPN as a causal loss in practice. 

\subsection{Using a CaSPN as Causal Loss}
To deploy a causal loss we are required to provide a CaSPN that approximates the target distribution $p_\mathcal{\hat{G}}(\mathbf{Y}_i = \mathbf{Y}_i|\mathbf{X}_i = \mathbf{x}_i, \mathbf{Z}_j = \mathbf{z}_j)$.
As defined previously, a CaSPN is an extension of an SPN. To adapt the SPN to conditional variables we make use of a neural network approximator to predict the weights of the individual sum and leaf nodes, as proposed in conditional-SPNs \citep{shao2020conditional}. For interventional distributions \cite{Zecevic2021a} proposed to pass the adjacency matrix of the intervened causal graph $\mathcal{\hat{G}}$ as a condition into the neural network. We follow a similar approach for CaSPNs and reintroduce the observed variables as conditionals, which means in practice, that we concatenate the adjacency matrix and observed variables and feed them into the neural network to predict the SPN weights. This yields a CaSPN which estimates the conditional probability $p^\mathfrak{C}_\mathcal{\hat{G}}(\mathbf{Y}|\mathbf{X}, \mathbf{Z}_j = {z}_j)$. The obtained CaSPN can then be used as a causal loss.


Figure~\ref{caSPNConstruction} presents the four steps of CaSPN construction (left to right). In the first step we split the variables of the SCM into conditional (marked in blue) and target variables (green). The conditional variables, together with the adjacency matrix, will be passed into the neural weight estimator at training time. The set of target variables will be applied to the SPN leaf nodes. This means, in return, that we must construct our SPN in a way that will approximate the set of target variables given the adjacency matrix and observed variables.

Since CaSPN are random SPNs, the model structure is determined by randomly generated region graph. Each region graph starts out with the root region which contains the set of all target variables. In a recursive process, starting at the root region, the variables at each region are split into a balanced partition of two mutually exclusive sets of variables. For both of the sets a new child region is created. Each of these new regions covers only the the variable splits. Both child regions attached to the current node. The process is repeated until only a single variable is left in a region or a predefined recursion depth is reached. For our CaSPN we define a maximum recursion depth of two.

In a third step the SPN structure is constructed from the region graph. For every leaf of the region tree, a leaf node with 4 Gaussian distributions is created. For each parent region multiple product nodes are created so that every leaf output from one split is multiplied by a leaf output from the other by one of the product nodes. Then, four sum nodes are created per region where each sum node computes the weighted sum of all product nodes of the region. This process of creating product and sum nodes is repeated until root region is reached, for which only a single sum node is created.

As a last step we generate training data from our SCM. For each SCM we sample 100,000 samples from the unintervened SCM as well as 100,000 samples for every variable intervention. All sampled interventions are perfect atomic, which means that only a single variable is intervened at a time and that this intervention cuts all edges to the parent nodes in the SCM. The intervened variable is set to a uniform distribution.


For training the CaSPN is trained to estimate the log likelihood of the data. The adjacency matrix and the observed variables of the samples are passed onto the weight estimator. With the predicted weights and given the real class configuration the CaSPN then predicts the log-likelihood. The optimizer then tries to minimize the negative log likelihood by adjusting the parameters of the CaSPN.

\subsection{Training with a Causal Loss}
In principle, we can train any predictive machine learning model with the proposed causal loss. We present the training process of 2 models, namely, neural networks and decision trees. We choose these models since they have contrasting properties. Decision trees are non-differentiable and white-box whereas neural networks are differentiable and black-box.

\begin{figure*}
\centering
\includegraphics[width=0.9\linewidth]{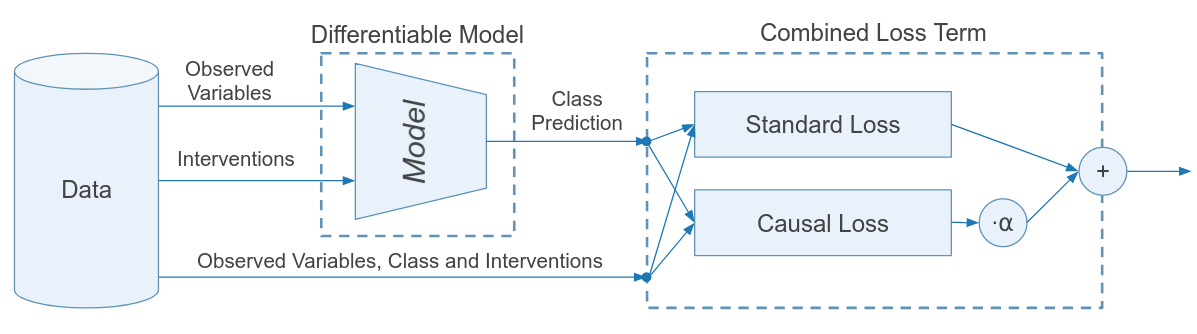}

\caption{Setup for neural network training with causal loss.}
\label{causalLossSetup}
\end{figure*}


\subsubsection{Training a Neural Network} 
Since CaSPNs are differentiable models, they can be used for training with any gradient descent optimizer. Using the causal loss and a standard loss, neural network training is performed in an end-to-end fashion as shown in Figure~\ref{causalLossSetup}. The observed variables and the mutilated graph, in form of an adjacency matrix, are concatenated and passed onto the model. The observed variables, intervention data and real class are passed along to the losses. While a standard loss such as cross entropy or mean squared error depend on the predicted and real class, the causal loss depends on the interventional information, observed variables and class prediction. Depending on the setup the losses can be used in combination or stand-alone. For combined training setup, the causal loss is multiplied by a weighting term $\alpha$.

When applying a causal loss for a classification, the causal loss predicts the highest likelihood when the model is predicting the most likely configuration of $\hat{Y}$ according to the distribution $p^\mathfrak{C}_\mathcal{\hat{G}}(\mathbf{Y}|\mathbf{X}, \mathbf{Z}_j = \mathbf{z}_j)$. For training with a causal loss we try to maximize $\mathfrak{L}^C(\hat{Y}, | \mathfrak{C}, X, Z)$.


\subsubsection{Training a Decision Tree} Using the causal loss with decision tree is a more involved process. For this purpose we present a new scoring function that utilizes a neural causal model. 




\textbf{Causal Decision Score.} The construction of decision trees requires a measure that indicates for each node which decision should be preferred for splitting the data. The Gini index is an often used indicator for measuring the (im)purity of a sample set, where an higher value indicates a better separation of the data.

Taking advantage of the neural causal density estimator at hand, we propose a Causal Decision Score (CDS) which measures to average probability of a split resulting in a correct classification. To do so we revisit the definition of CaSPNs and remember that they estimate the conditional probability $p^\mathfrak{C}_\mathcal{\hat{G}}(\mathbf{Y}|\mathbf{X}, \mathbf{Z}_j = \mathbf{z}_j)$. For each node in the tree we are given a set of samples $(\mathbf{X},\mathbf{Y})$. The likelihood of a correct prediction at each given node will be maximized if we predict the most probable class $\textbf{C}$ of $\mathbf{Y}$. Assuming a two-fold split of the data according to some criterion $S_j(x, y) \in \{0, 1\} | j \in {X_0, ..., X_N}$, where 1 indicates that $S_j(x, y)$ is  satisfied and 0 indicates otherwise. Every variable $X_j \in \textbf{X}$ is considered as a criterion $S_j(x, y)$ and thereby we obtain new exclusively disjunct subsets:
\[
\mathbf{X}_j := \{(x_i,y_i) \in (\mathbf{X},\mathbf{Y}) | S_j(x_i,y_i)\}
\]

where a data point $(x, y)$ is included in subset $\mathbf{X}_0$ iff $S_j(x, y) = 0$ and in subset $\mathbf{X}_1$ iff $S_j(x, y) = 1$. For each of the two subsets $(\mathbf{X}_i,\mathbf{Y}_i) | i \in {0, 1}$ a new most probable class $\textbf{c}_{i,j}$ is predicted.


For building decision trees, we would like our score to prefer the criterion $S(x, y)$ for data splitting that maximizes the probability of the data split $\mathbf{X}_i$ given the predicted most probable class $\textbf{c}_i$.

\begin{defi} (Causal Decision Score)
 A Causal Decision Score $D(\textbf{X}, c_{j}) \in {0, 1}$ estimates the average likelihood of a set of data $\textbf{X}, \textbf{Z}$ according to a causal loss, given the class predictions $\textbf{c}$ of each sample: \[D(\textbf{c}, \textbf{X}, \textbf{Z}) := \frac{1}{N} \sum_{n = 0}^{N} \mathfrak{L}^C(\hat{c_n}, | \mathfrak{C}, X_n, Z_n) \] $N$ is the number of samples and $\mathfrak{L}^C(\hat{c_n}, | \mathfrak{C}, X_n, Z_n)$ is the causal loss.
\end{defi}


Choosing the criterion $S_{\hat{j}}(x, y)$ such that $\hat{j} := \argmax_j D(\textbf{X}, c_j)$ where $x_i \in \textbf{X}_j$ iff $S_j(\textbf{X}, \textbf{Y}) = 1$ maximizes the weighted average probability $p(\textbf{X}_j|\textbf{Y}_j)$ for a correct classification of the split.

While the Gini index aims to improve the purity of the data sub-splits, CDS aims to improve the probability of the data given the most probable class. This allows the incorporation of the underlying causal intervention information of the data into the score term and therefore lift the CDS to rung two of the Pearl Causal Hierarchy. Since we are interested in the quality of the resulting tree from our neural causal scoring function, we will inspect the tree right after its initial construction, thereby avoiding potential artifacts in the tree structure resulting from further refinements.

\section{Empirical Evaluation}
Through our extensive empirical evaluations, we aim to answer the following research questions: 
\begin{enumerate}
    \item[] \textbf{Q1.} How effective is the proposed causal loss for training discriminative machine learning models and obtain correct predictions?
    \item[] \textbf{Q2:} Can a combination of standard loss function(s) and causal loss result in more discriminative models? 
    \item[] \textbf{Q3:} Can we train contrasting set of models i.e. white-box as well as black-box models using causal loss?
    \item[] \textbf{Q4:} What does  decision tree learning reveals about the training with causal loss?
\end{enumerate}
We present the application of causal loss for two classification settings and compare the performance of causal loss in different stand-alone and combined settings and show that the resulting models meet the performance of standard trainings. In a second setting we display that causal loss can help to improve linear classifiers which is manifested in the models' performance as well as in the model structure.

\subsection{Experimental protocol}

For the first experiment a multi-layer perceptron (MLP) is trained on the ASIA \citep{Lauritzen1988}, CANCER, EARTHQUAKE \citep{korb2010bayesian}. Additionally we will test the performance on our newly deviced Causal Health Classification (\textit{CHC}) dataset, which adds three Bernoulli \textit{Diagnosis} variables in a one-hot configuration to the existing Causal Health Dataset \citep{Zecevic2021a}. For each of the first three datasets we pick two of the random variables as targets \footnote{Target variables for the different datasets: ASIA (X,D); EARTHQUAKE (J,M); CANCER (X,D); CHC ($D_1,D_2,D_3$)}. For CHC the target variables are the newly added diagnosis variables. Please refer to the appendix for the SCM and structural equations of the new dataset.

\begin{table*}[ht]
	\caption{Comparison of classification accuracies of CaSPNs and NNs trained on ground truth data using a standard loss (Cross entropy for CHC and MSE for ASIA, CANCER and EARTHQUAKE) and NNs using a pure causal loss or a combinations of standard and causal loss. Trained with Adam over 80 epochs, learning rate 1e-5. Best results are shown in bold. For $\mathbf{S} + \alpha \cdot \mathbf{C}$ the best performing $\alpha$ is shown as subscript.}
	\begin{center}
		\begin{tabular}{l|cc|cccccc}
\toprule
\multicolumn{1}{r|}{\textbf{Model}} & \multicolumn{2}{c|}{\textbf{caSPN}} & \multicolumn{6}{c}{\textbf{Neural Network}} \\
\multicolumn{1}{r|}{\textbf{Loss}}  & \multicolumn{2}{c|}{} & \multicolumn{2}{c}{\textbf{Standard Loss (S)}} & \multicolumn{2}{c}{\textbf{Causal Loss (C)}} & \multicolumn{2}{c}{$\mathbf{S} + \alpha \cdot \mathbf{C}$} \\
\midrule
\textbf{Data set}          & Mean  & Std  & Mean  & Std  & Mean  & Std   & Mean  & Std   \\
\midrule
Causal Health Class. & 80.19 & 0.13 & 79.92 & 0.05 & 78.17 & 0.08  & \textbf{80.13}\textsubscript{0.01} & 0.06  \\
ASIA                 & 83.97 & 1.50 & 82.90 & 0.07 & \textbf{84.72} & 0.05 & \textbf{84.72}\textsubscript{0.02} & 0.06  \\
CANCER               & 56.34 & 0.13 & \textbf{56.34} & 0.13 & 52.61 & 0.09  & 56.33\textsubscript{0.001} & 0.13  \\
EARTHQUAKE           & 85.13 & 0.05 & \textbf{85.13} & 0.05 & 74.88 & 0.07  & \textbf{85.13}\textsubscript{0.001} & 0.05  \\
\bottomrule
		\end{tabular}
	\end{center}
	\label{tab:neural_net}
\end{table*}

\begin{table*}[ht]
	\caption{Accuracy evaluation for decision trees constructed on different datasets for altering scoring functions. "Gini Index scikit" indicates a DT learned by using the scikit-learn library, which makes use of a more complex approach base on post-pruning than our baseline algorithm. We also tried the combination loss for DTs but it shows no improvement.}
	\label{tab:dtrees}
	\begin{center}
		\begin{tabular}{l|llll|ll}
\toprule
{}     & \multicolumn{2}{c}{\textbf{Gini Index}} & \multicolumn{2}{c|}{\textbf{Causal Loss Score}} & \multicolumn{2}{c}{\textbf{Gini Index scikit}} \\
\textbf{Dataset}     & Mean  & Std  & Mean  & Std  & Mean  & Std   \\
\midrule
ASIA                 & 44.56 & 0.10 & \textbf{84.72} & 0.06 & \textbf{84.72} & 0.13 \\
CANCER               & 48.72 & 0.17 & \textbf{56.34} & 0.13 & \textbf{56.34} & 0.13 \\
EARTHQUAKE           & 74.52 & 0.12 & \textbf{85.13} & 0.05 & \textbf{85.13} & 0.05 \\
\bottomrule
		\end{tabular}
	\end{center}
\end{table*}
\textbf{General setup}. The neural estimator for the CaSPN weights $\psi$ is a MLP. The MLP is constructed of a 'base' consisting of 3 linear layers with ReLU activations \citep{nair2010rectified} and final dropout layer \citep{JMLR:v15:srivastava14a}. The result is fed into two heads that output the final leaf and sum weights for the CaSPN. The heads repeat the structure of the base except that we leave out the dropout and ReLu after the last layer.

As reference neural network (NN) baselines in our experiments we use the same architecture as the CaSPN weight estimator. For classification we only attach one head that directly predicts each binary variable in a one-hot encoding.

For optimizing the models in this section we make use of the Adam optimizer \citep{kingma2014adam} with a start learning rate of \num{1e-5}. The value was obtained by probing multiple learning rates manually. For a quantitative consideration of the results every experimental setup is carried out over 80 epochs for five different seeds.

\textbf{Creating intervened datasets}. In consideration of creating a causal environment, we create data by sampling from the SCMs with and without interventions. We consider "perfect atomic" interventions on all conditional variables. Where "atomic" means that an intervention is carried out on a single variable at a time. And "perfect" indicated the replacement of the structural equation of with a uniform distribution, which cuts all causal dependency on the parent variables.

For a model to be able to answer interventional queries according to the PCH, it needs to incorporate interventional information (\citep{bareinboim2020pearl}, Cor. 1). Subsequently training with a causal loss can only induce causal information into the trained model \textit{iff} provided with interventional information. Being demonstrated as a feasible solution for iSPNs \citep{Zecevic2021a}, we provide the adjacency matrix of the causal graph together with the conditional variables to the weight estimator of the CaSPN. For setups with direct NN classification the adjacency matrix and conditional variables are fed directly to the NN.

\begin{figure*}
\subfloat{{\includegraphics[width=0.24\linewidth]{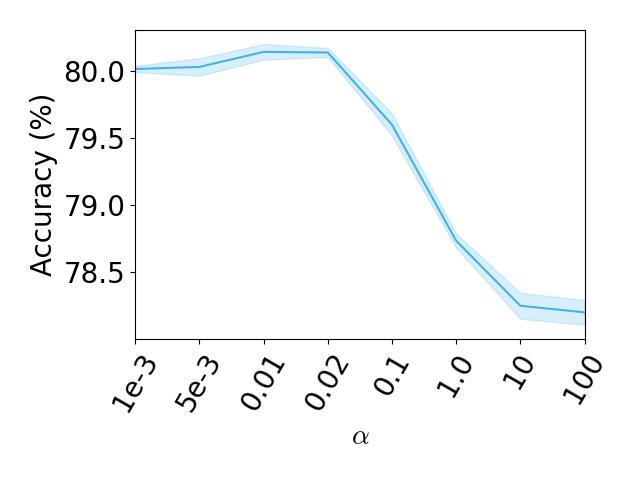}}}
\subfloat{{\includegraphics[width=0.24\linewidth]{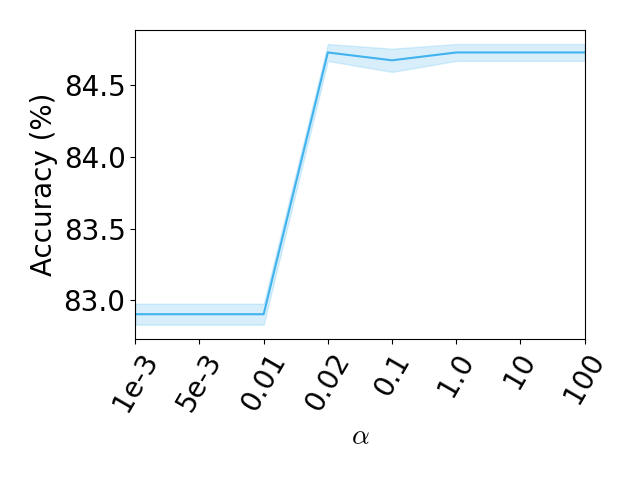}}}
\subfloat{{\includegraphics[width=0.24\linewidth]{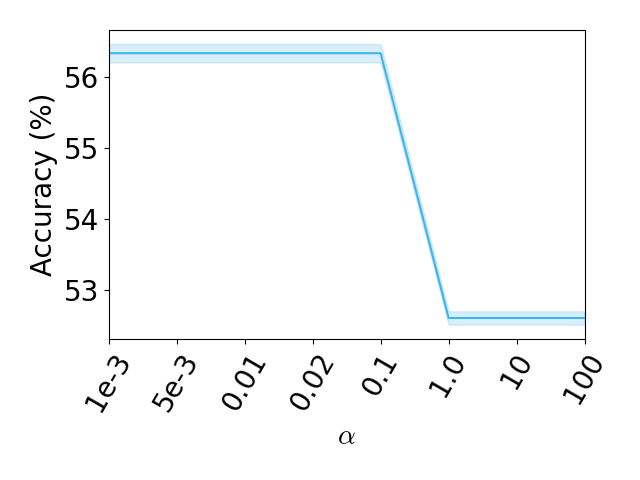}}}
\subfloat{{\includegraphics[width=0.24\linewidth]{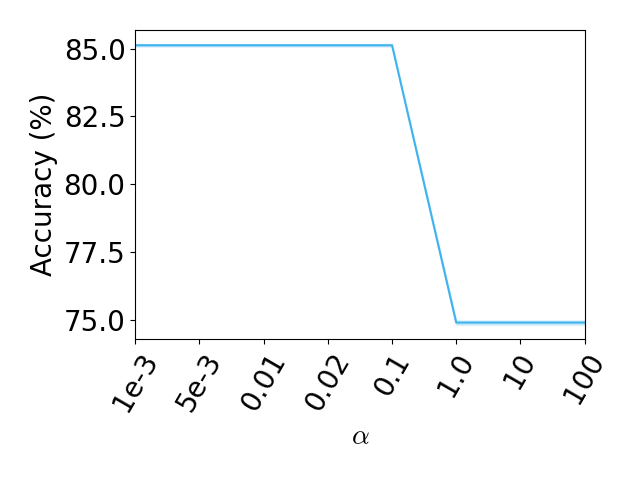}}}
\vspace*{-4mm}
\centering
\caption{Training of a neural network using a combined loss of cross entropy (for CHC) or MSE (for ASIA, CANCER, EARTHQUAKE) plus a causal loss. Left to right are the plots for CHC, ASIA, CANCER and EARTHQUAKE. $\alpha$ indicates the factor by which the causal loss is added to the CE or MSE.}
\label{combinedLossTrainingCHC}
\end{figure*}

\begin{figure*}
\centering
\vspace*{-3mm}
\subfloat{{\includegraphics[width=0.25\linewidth]{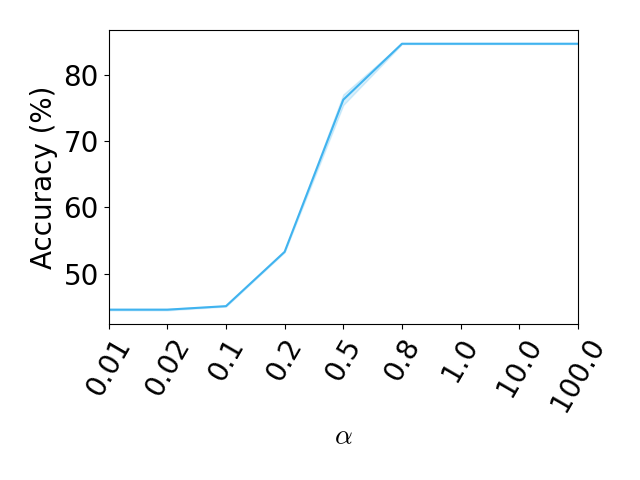}}}
\hspace{0.03\linewidth}
\subfloat{{\includegraphics[width=0.25\linewidth]{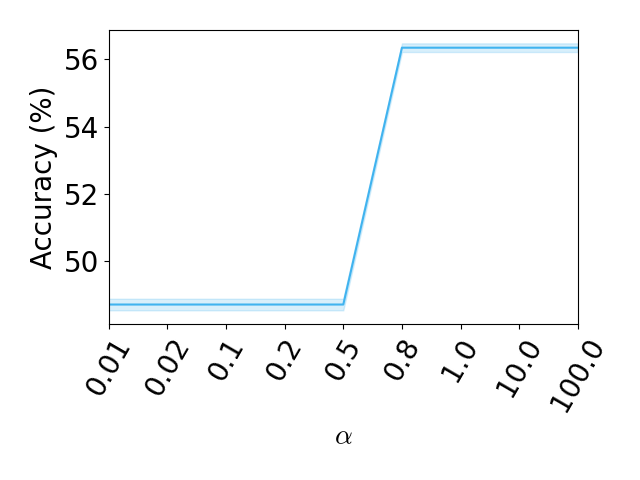}}}
\hspace{0.03\linewidth}
\subfloat{{\includegraphics[width=0.25\linewidth]{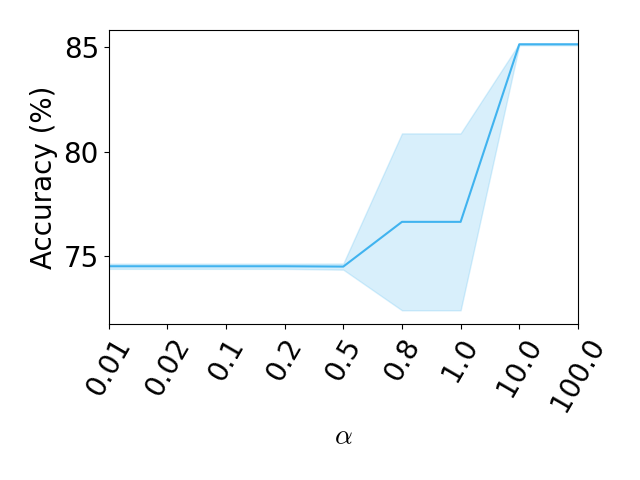}}}
\vspace*{-4mm}
\caption{Accuracies of Decision Trees using a combined loss of the Gini index plus our causal decision score. $\alpha$ indicates the factor by which the causal loss is added to the Gini index.}
\label{combinedDTTraining}
\vspace{-0.2in}
\end{figure*}
\textbf{(Q1. Causal Loss effectiveness for prediction)} In this setup we aim to maximize the accuracy of CaSPNs $m(\textbf{G}, \textbf{D}, \theta)$ and NNs $n(\theta)$ by optimizing the weights $\theta$ of the respective models. To test our causal loss we proceed to train a neural network (considered baselines interventional SPNs and a neural network trained with a standard loss) and a decision tree (considered baselines decision tree with Gini index) with a causal loss and compare it to various baselines.

Table \ref{tab:neural_net} shows the results for training a neural network with causal loss. It can be seen that the neural network trained with the causal loss performs comparably to the baselines and outperforms in the ASIA data set. We observe that CaSPNs with cross entropy is performing best in 3 out of the 4 data sets. We notice that the results on CANCER and EARTHQUAKE are the same for CaSPN and standard neural network. We expect the models learn the optimal policy, given the simple structure and random nature of the SCM. 

Table \ref{tab:dtrees} shows the results for training a decision tree with the causal decision score. A pure inspection of the model performance reveals that our baseline algorithm with a causal loss score achieves high accuracies across all data sets, learning the optimal decisions. Decision tree construction with Gini index is not able to reach the performance of our causal loss score models. Due to the high accuracy of our causal decision score we observe the best performance is gained when giving a large weight to the causal loss or using causal loss only. 

Note that we use  two implementations of decision tree: 1 from scikit-learn library in python\footnote{https://scikit-learn.org/stable/modules/tree.html} (represented as Gini Index scikit in Tab.\ref{tab:dtrees}) and other our own (represented as Gini Index). This is necessary due to the fact that scikit decision tree uses compiled python code and it is a tedious process to add the proposed causal loss. We can see that the (represented as Gini Index scikit matches the performance of the causal loss but involves some extra steps such as pruning and uses an optimized CAR decision tree learner. The fact that our causal loss function is able to reach the optimum performance without requiring these extra steps shows that training with causal loss i.e. using causal decision score as the splitting criteria is highly effective.

Using a combinations of Gini index and causal decision score might still be of interest for the scenario of partial unknown SCMs, as the graphs show a smooth transition between the Gini index and Causal Decision Score accuracies for all three data sets.

Thus, the causal loss provides sufficient feedback to the model during training to achieve good performance in causal classification setting answering \textbf{Q1}: causal loss is effective in obtaining correct predictions while also driving correlation based models towards causality.

\begin{figure*}[t]
\includegraphics[width=0.4\textwidth]{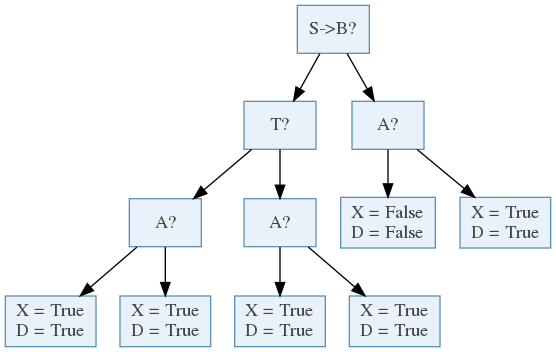}
\includegraphics[width=0.5\textwidth]{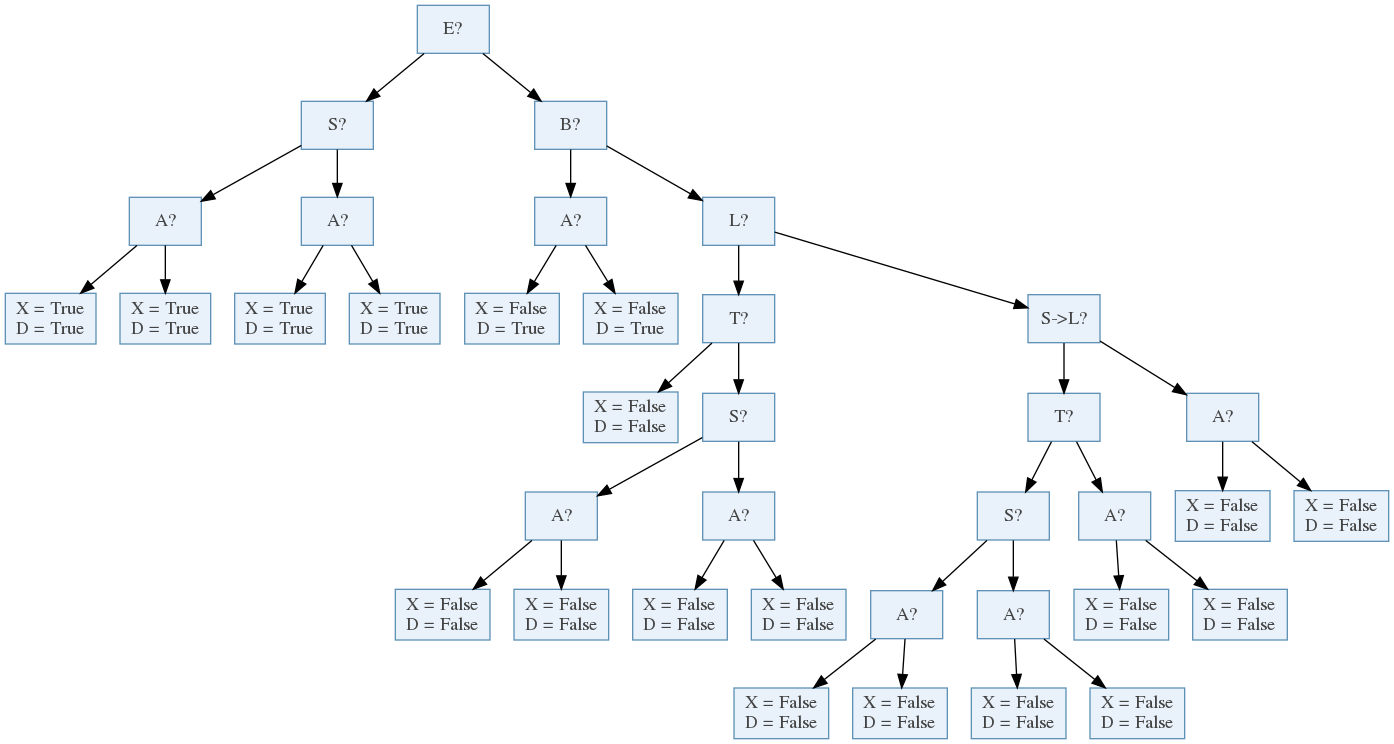}
\includegraphics[width=\textwidth]{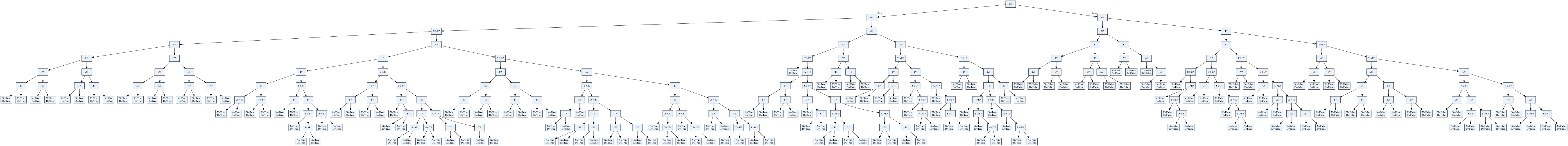}
\caption{Decision trees constructed for the ASIA dataset using Gini index (left) and Causal Loss (right) with our baseline algorithm. The tree constructed with scikit (bottom) reaches the accuracy of our causal loss tree but needs to be much larger in size. Each node either marks the decision over a boolean variable "X?" or the existence of an edge from the adjacency matrix "X$\rightarrow$Y?". Leaf nodes contain the final class decision.}
\centering
\label{dtLearning}
\vspace{-0.2in}
\end{figure*}

\textbf{(Q2. Combining causal loss with standard loss)} We study the combined usage of the standard loss functions, cross entropy for the CHC dataset and the mean squared error for ASIA, CANCER, EARTHQUAKE, and  causal loss during training. We conduct a series of experiments where the causal loss is added to the cross entropy by some factor $\alpha$: 
$\textbf{CombinedLoss}_\alpha := \textrm{StandardLoss} + \alpha * \textrm{CausalLoss}.$

The training setup displayed in Figure~\ref{causalLossSetup} is used, now with combined loss and a varying $\alpha$. We choose the value of $\alpha \in \{1e-3, 5e-3, 1e-2, 2e-2, 0.1, 1.0, 10.0, 100.0\}$ using grid search and report the mean accuracy and standard deviation. Table \ref{tab:neural_net} present the results of neural network trained with the combined loss function. Note that we also tried the combination loss for decision tree learning but it shows no performance improvement. With an increasing weight $\alpha$ of the causal loss we observe a transition in accuracy from that of the neural network baseline to the performance of the CaSPN. We report the results with the best $\alpha$. It can be seen that the model trained with the combined loss function is either better or comparable to the considered baselines. From this experiment it can be concluded that, even for already strong baseline models, the addition of a causal loss to the training process may improve and more importantly does not degrade the model performance. Figures~\ref{combinedLossTrainingCHC} and \ref{combinedDTTraining} shows no synergetic effect for using a combination of Gini index and causal decision score for both decision tree with varying $\alpha$. This answer \textbf{Q2}.

\textbf{(Q3. Training contrasting models.)}. The aim of this experiment is to highlight the fact that the causal loss can be used to improve the predictive modeling power of differentiable but black-box models such as a neural network and non-differentiable but interprtable models such as a decision tree. Driving correlation to imply causation while training these models can provide an advantage of utilization of inherent causal relationships in the data. This can result in the learning of more powerful predictive models while also respecting the causal realm. This answers \textbf{Q3} affirmatively.  



\textbf{(Q4. Lessons from learned decision trees)} 
We analyse the ability of our causal scoring function to construct decision trees. Figure~\ref{dtLearning} shows three decision trees taken from our experiments. The left one is constructed with Gini index, while the right one is built with the proposed causal decision score. In general we observe that our causal decision score builds more a more sophisticated decision tree than Gini index as it makes use of the underlying causal density information of the data. This is also reflected in the predictive accuracies obtained by both sets of decision trees.

An important point to consider here is the comparison of the constructed tree while using the causal loss to the tree constructed from Gini index with scikit (Fig. \ref{dtLearning}(bottom). Although, as the results show that the prediction accuracy of the scikit decision tree is on par with the causal loss decision tree, it is very large. This shows that, using a standard splitting criteria, it requires a much more complex decision tree to obtain the same prediction accuracy as we get learning a very compact tree using the CDS. This answers \textbf{Q4}.

\section{Conclusions}


We presented the causal loss, a simple yet highly effective way to drive correlation-based models towards causality. The main idea is to make use of what we called causal sum-product networks (CaSPNs) that estimate interventional distributions of structural causal models conditioned on observational data. The inclusion of interventional information distinguishes CaSPNs from correlation based models and places them on rung two of the Pearl causal hierarchy.
More importantly, CaSPNs can directly be used as causal loss with any differentiable and even non-differentiable ML model. In doing, they provide meaningful feedback to, e.g., neural networks for classification tasks. Furthermore, causal losses can be used to construct well performing decision trees with a smaller tree sizes than competing algorithms.

While a causal loss can be used on its own for model training, we found in experiments with neural networks that adding a causal loss to a classical training setup improves actually performance. With the introduced weighting term a user can steer the training towards are more causal or correlation based driven model according to the underlying requirement.

A drawback of the causal loss is its dependence on a complete SCM which might not be known for real world data. With a combined loss training the user can balance between the standard and causal loss. However, finding the optimal weighting requires an additional parameter search.





There are several interesting avenues for future work. The CaSPNs act on level two of the Pearl Causal Hierarchy and an interesting direction for enhancement of the causal loss is to lift it onto level three i.e. making it counterfactual. Given such a counterfactual loss, one could imagine to adapt model training to distributions that are not covered by the data set distribution itself. Moreover, one should scale 
learning of decision trees using a causal loss to very large data sets. Here one should also investigate more whether causal decision trees are actually easier to understand for humans due to their causal nature.  


\subsubsection*{Acknowledgements}
This work was supported by the ICT-48 Network of AI Research Excellence Center “TAILOR" (EU Horizon 2020, GA No 952215) and by the Federal Ministry of Education and Research (BMBF; project “PlexPlain”, FKZ 01IS19081). It benefited from the Hessian research
priority programme LOEWE within the project WhiteBox, the HMWK cluster project “The Third Wave of AI.” and the Collaboration Lab “AI in Construction” (AICO).

\bibliography{causalLoss}

\clearpage
\thispagestyle{empty}
\onecolumn
\aistatstitle{Supplementary Materials for\\ The Causal Loss: Driving Correlation to Imply Causation}
\vskip -0.2in plus -1fil minus -0.1in
\setcounter{section}{0}

This supplementary material provides the construction and analysis of the Causal Health Classification data set as well as an inspection of the loss curves of the model trainings from the paper.

The source code accompanying this submission is provided under \url{https://anonymous.4open.science/r/causalLoss-E5F2/}.

\section{The Causal Health Classification Data set}

The causal health classification data set is an extension to the causal health data set by \cite{Zecevic2021a}, where three new binary "Diagnosis" variables, in a one-hot configuration, are introduced. In this case one of the diagnose variables is set to true and the remaining ones are set to false for every data point.

\begin{figure}[h]
\centering
\includegraphics[width=0.2\linewidth]{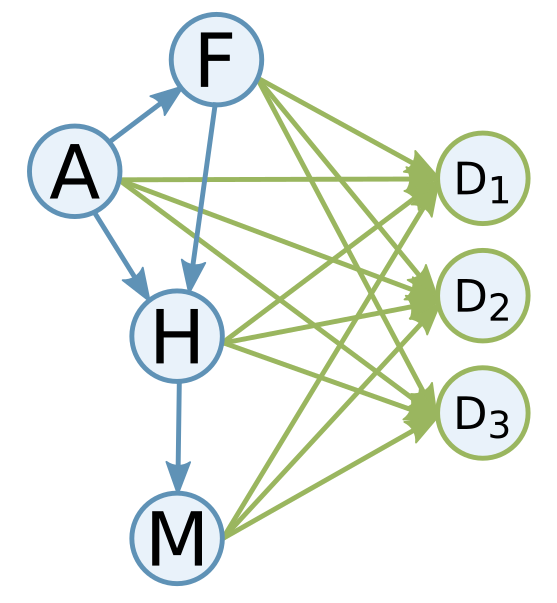}
\caption{The structural causal model of the Causal Health Classification Data set. Variables of the already existing Causal Health data set (blue) with the new diagnosis variables and causal connections (green) added.}
\label{SCMCHC}
\end{figure}

For each sample three multivariate polynomial functions are evaluated to determine the activate diagnose. Each function which depends on a subset of the original variables $A, M, F$ and $H$:
\[
\begin{split}
  & f_1(A) :=
  \begin{cases}
    0.00108 A^3 - 0.08862 A^2 + 1.337 A + 30 + N(0, 10) & \text{if } A\leq45.667\\
     9.09837 + N(0, 10),              & \text{otherwise}
  \end{cases} \\
  & f_2(F,M) := 0.0175F + 0.525M + N(0, 5) \\
  & f_3(A,H) := 0.00013857 A^3 - 0.0135 A^2 + 0.2025 A + 0.2025 H + 17.1714 + N(0, 0.2A)
\end{split}
\]

The state of each diagnose variable $D_i$ is determined by taking the argmax over all three functions.
\[
  f_{Di}(A, F, H, M) := 
  \begin{cases}
    \mathit{true} & \text{if} \argmax(f_1(A), f_2(M,F), f_3(A,H)) = i \\
    \mathit{false} & \text{otherwise}
  \end{cases}
\]

The resulting SCM, as shown in Figure~\ref{SCMCHC}, consists of the Causal Health SCM with three additional Diagnose variables. Connections from $A, F, H$ and $M$ to every $D_i$ are introduced, since $f_{Di}(A, F, H, M)$ depends on all four original variables.

\newpage

The resulting diagnose distribution on the SCM without interventions with all but one of $A, F, H$ and $M$ marginalized out is displayed in Figure~\ref{chcDensity}. The distribution of $D_1$ (blue) is observed to feature a mode which clearly is distinct from $D_2$ (green) and $D_3$ (orange), while $D_2$ and $D_3$ overlap on a larger value range. The overall diagnose probabilities result to 19.91\%, 37.43\% and 42.64\% for $D_1$, $D_2$ and $D_3$ for the whole data set generated from 100.000 samples.

\begin{figure}[h]
\centering
\includegraphics[width=0.45\linewidth]{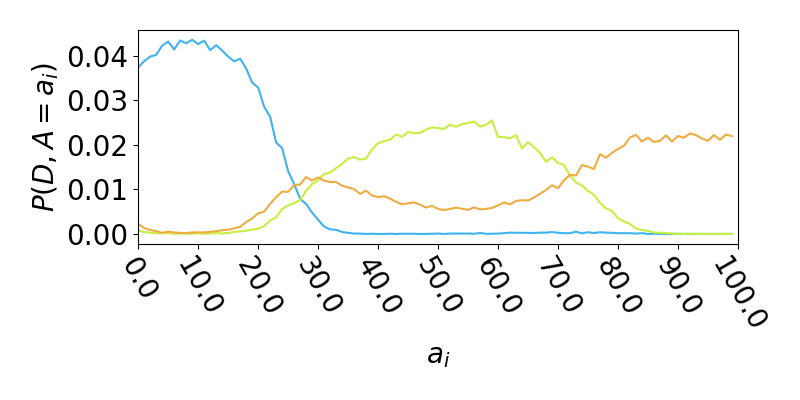}
\includegraphics[width=0.45\linewidth]{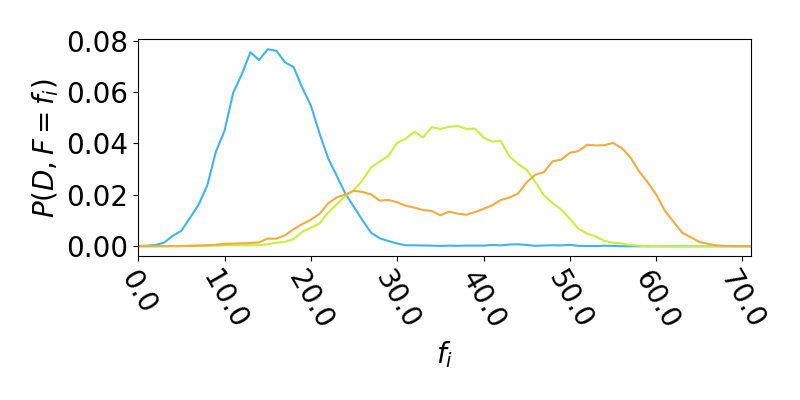}
\\
\includegraphics[width=0.45\linewidth]{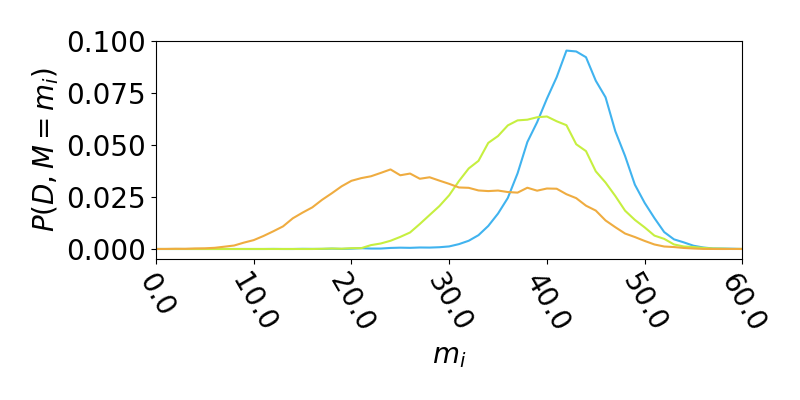}
\includegraphics[width=0.45\linewidth]{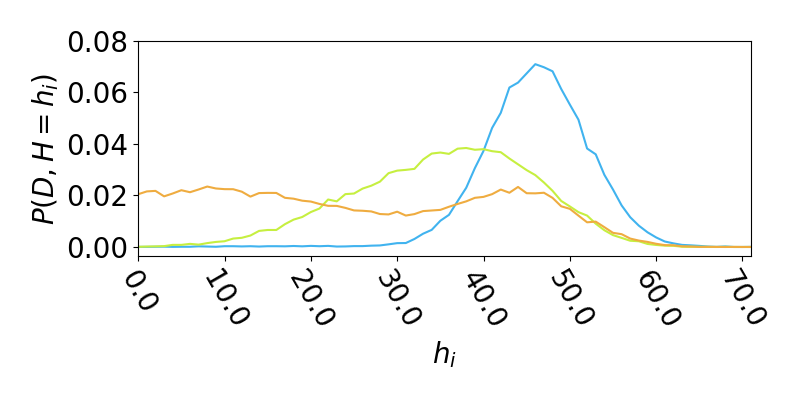}

\caption{Diagnose probability $p(D, V_i=v_i)$ for Causal Health Classification per variable $V_i \in \{A,F,H,M\}$. Class densities for $D_1$ (blue), $D_2$ (green), $D_3$ (orange) are displayed for a data set sampled from the structural causal model without interventions (n=100.000 samples).}

\label{chcDensity}
\end{figure}

\begin{figure}[b]
\centering

\includegraphics[width=0.19\linewidth]{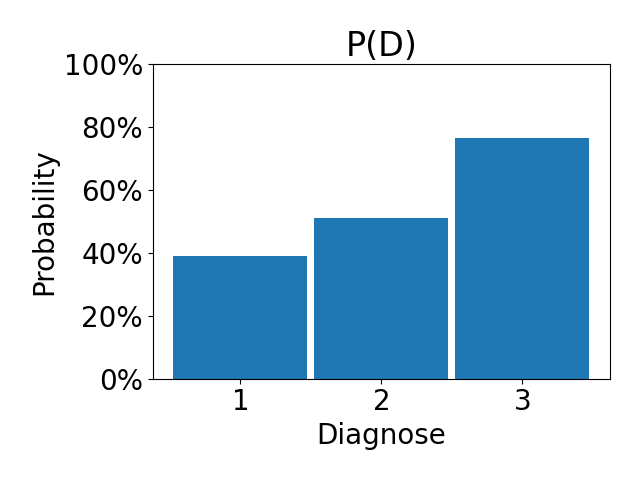}
\includegraphics[width=0.19\linewidth]{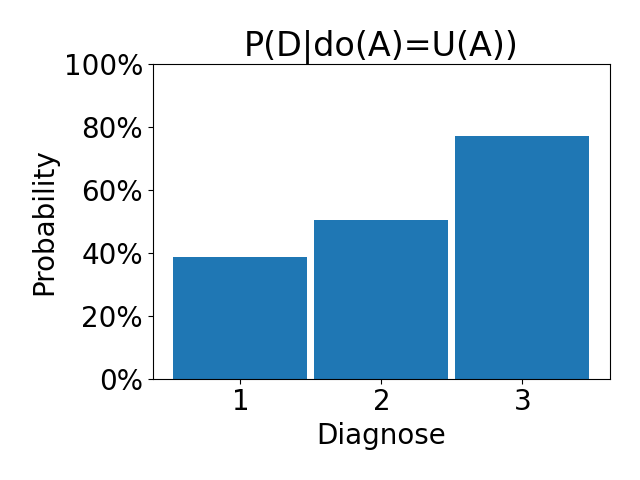}
\includegraphics[width=0.19\linewidth]{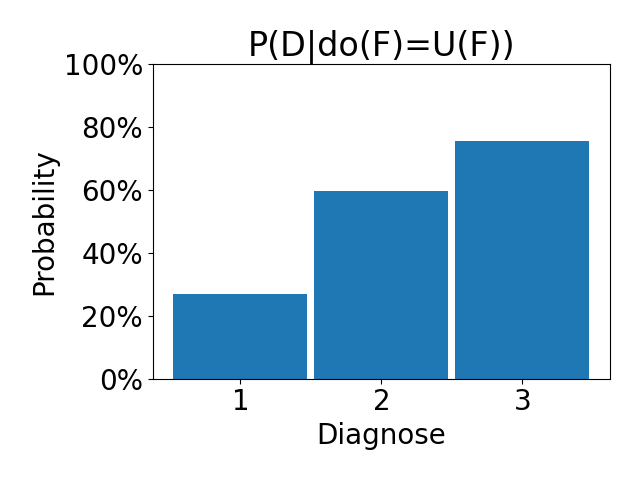}
\includegraphics[width=0.19\linewidth]{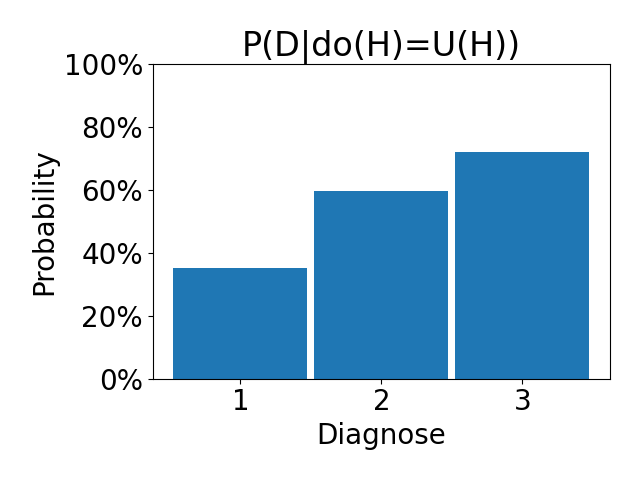}
\includegraphics[width=0.19\linewidth]{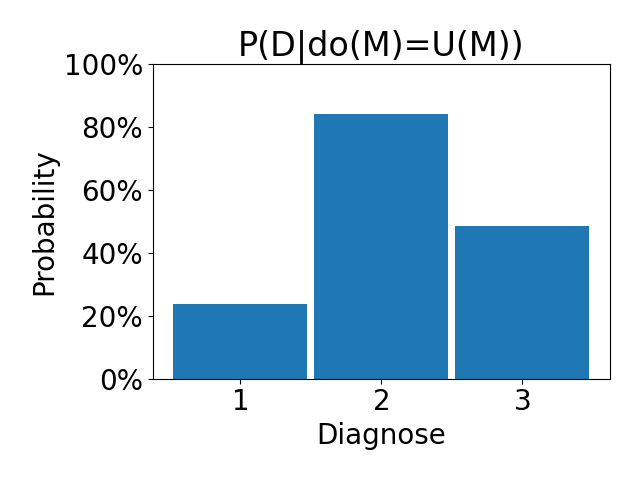}

\caption{Diagnose class distributions for Causal Health Classification data under interventions.}
\label{chcDensityInter}
\end{figure}

The effects of interventions on the diagnose distribution is shown in Figure~\ref{chcDensityInter}. All interventions are carried out to be perfectly atomic, while every intervention sets the affected variable to a uniform distribution.

\section{Learning curves}

\begin{figure}
\centering

\subfloat{{\includegraphics[width=0.38\linewidth]{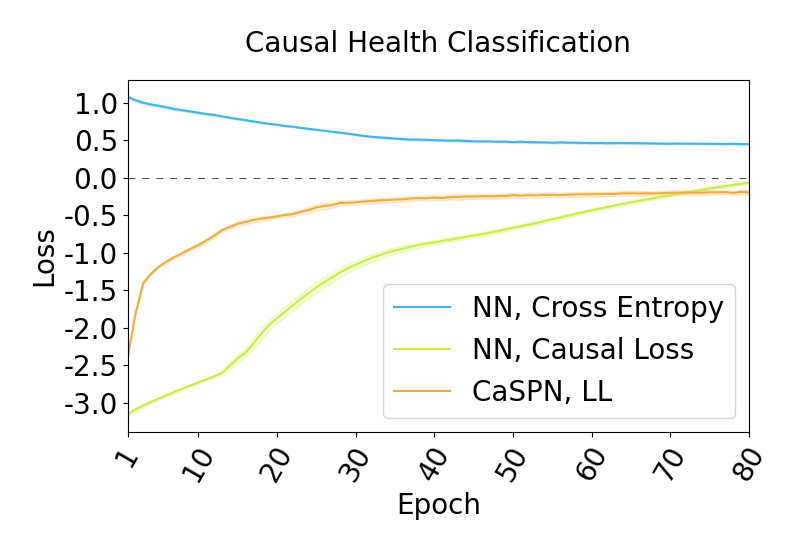}}}
\subfloat{{\includegraphics[width=0.38\linewidth]{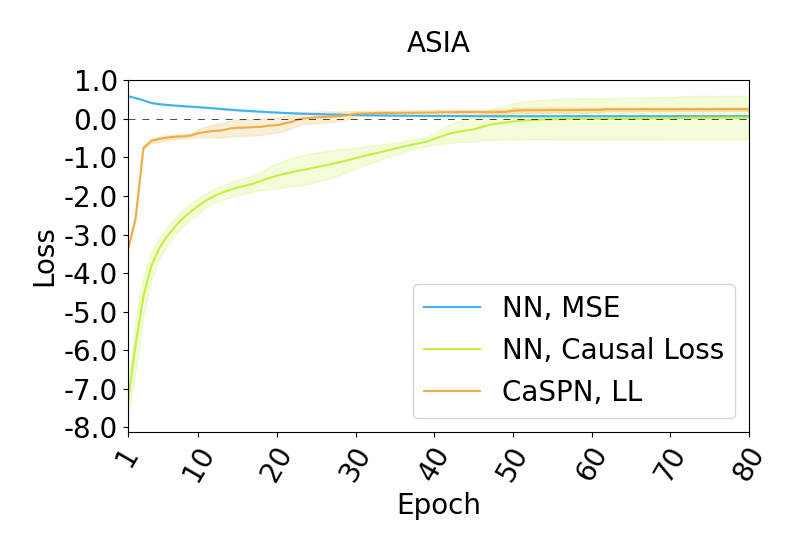}}}
\\
\subfloat{{\includegraphics[width=0.38\linewidth]{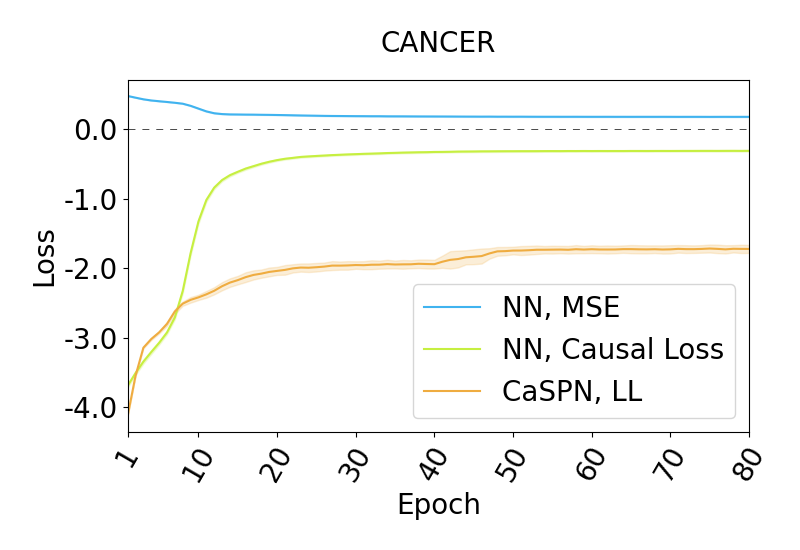}}}
\subfloat{{\includegraphics[width=0.38\linewidth]{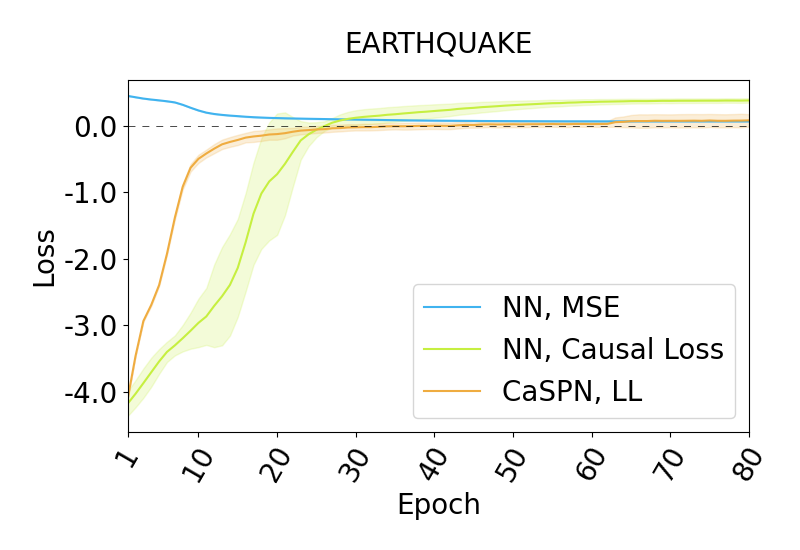}}}

\caption{Loss curves for neural net and CaSPNs trainings. During training the model parameters are optimized to minimize the cross entropy and mean-squared error, while trainings with Causal Loss and log-likelihood are trying to maximize the loss.}

\label{lossCurves}
\end{figure}

Figure~\ref{lossCurves} shows the loss curves of CaSPN training with log-likelihood and neural network trainings with cross entropy, mean-squared error and Causal Loss. We see a faster convergence for CaSPN and NN with cross entropy and MSE trainings, while trainings with neural networks and a Causal Loss tend to converge slower.

Due to the imbalanced class distributions on the ASIA and EARTHQUAKE data sets the CaSPN prediction tends towards a point estimate of the most likely class. This is revealed by the loss curves of the CaSPN predicting a high density for ASIA and EARTHQUAKE. Our CaSPN models the distribution at the leafs using Gaussians. Therefore a density and not a discrete probability is approximated. The standard deviation of the Gaussian distributions in the leaf nodes is bounded by a minimum value of $0.1$ while the most likely mode $M$ is predicted by the weight estimator for the observed data given the intervention information. The CaSPN then approximates the density of the sample point with $\mathcal{N}_{M,0.1}(M) \approx 1.26$ leading to the observed average log-likelihood of $\approx0.23$ in the loss curve.

\end{document}